\documentclass[11pt]{article}
\usepackage[T1]{fontenc}
\usepackage[utf8]{inputenc}
\usepackage{lmodern}
\usepackage[margin=1in]{geometry}
\usepackage{amsmath,amssymb,amsthm}
\usepackage{graphicx,booktabs,tabularx,array}
\usepackage{enumitem}
\usepackage[round,authoryear]{natbib}
\usepackage{xcolor}
\definecolor{linkblue}{RGB}{25,72,103}
\usepackage[colorlinks=true,allcolors=linkblue]{hyperref}
\usepackage{microtype}
\usepackage{fancyhdr}
\setlist[itemize]{topsep=3pt,itemsep=2pt}
\setlist[enumerate]{topsep=3pt,itemsep=3pt}
\newtheorem{proposition}{Proposition}
\newtheorem{assumption}{Assumption}
\newcommand{\E}{\mathbb{E}}
\newcommand{\Var}{\operatorname{Var}}
\newcommand{\Cov}{\operatorname{Cov}}
\newcommand{\tr}{\operatorname{tr}}

\newcommand{\clip}{\operatorname{clip}}
\newcommand{\credo}{\textsc{Credo}}
\title{\textbf{CREDO: Variance-Guided Rubric Evolution\\
for Replay-Corrected Credit Assignment}}
\author{Xuchun Hu}
\hypersetup{
  pdfauthor={Xuchun Hu},
  pdftitle={CREDO: Variance-Guided Rubric Evolution for Replay-Corrected Credit Assignment},
  pdfsubject={Preliminary method, theoretical analysis, and exact finite-model validation}
}
\date{September 2026}

\begin{document}
\maketitle
\begin{center}
\small\textit{Method and theoretical analysis. Full language-agent comparisons are pending.}
\end{center}
\begin{abstract}
Long-horizon language agents receive sparse terminal feedback, while intermediate
rubrics provide structured but potentially misspecified assessments of progress.
In resettable training environments, counterfactual continuation rollouts can
measure local credit, but exhaustive replay is costly. We propose \credo{}, a
framework that couples evolving semantic rubrics with selective, execution-based
credit correction. A frozen judge maps visible transitions to rubric features,
and a credit head predicts the change in expected terminal reward associated
with the realized transition. Independently sampled two-sided replays correct
prediction residuals using their recorded inclusion probabilities. We derive
conditional unbiasedness and a variance decomposition that connects two design
choices: which rubric features to retain, and where to allocate a fixed expected
replay budget. The resulting criterion weights prediction errors by policy-score
sensitivity and missing replay coverage; its allocation rule additionally
accounts for continuation cost. We also describe a practical mixture with
terminal leave-one-out advantages and distinguish its clipped, token-normalized
PPO implementation from the ideal policy-gradient estimator. This preliminary
report provides the method, proofs, an exact finite-model audit, and a controlled
evaluation protocol. It makes no claim of empirical superiority on language-agent
benchmarks.
\end{abstract}

\section{Introduction}
An agent can issue a valid API call, retrieve relevant information, and still fail
its task because it omitted a later constraint. A terminal success signal identifies
the failed trajectory, but does not explain which decisions should change.
Conversely, a rubric that rewards \emph{finding a candidate item} may systematically
overvalue that action when the agent has not checked the required size.
Both observations motivate local credit assignment, but a plausible explanation
of progress is not necessarily an accurate estimate of an action's contribution.

Group-based policy optimization offers a simple terminal-feedback baseline
\citep{shao2024deepseekmath,ahmadian2024back}.
Step-level grouping and semantic assessments provide additional structure
\citep{feng2025gigpo,gandhi2026draco}.
In an executable training sandbox, another source of information is available:
restore an intermediate state and observe fresh continuations
\citep{he2026branching}. Replaying every decision, however, can cost substantially
more than collecting the original trajectories. The question is therefore not
only how to predict credit, but \emph{which prediction errors deserve scarce
execution evidence}.

We study this question in finite-horizon environments with a trusted terminal
evaluator and reliable state restoration. \credo{} uses a frozen judge to score
visible transitions against an evolving set of semantic criteria. A small head
maps these scores to a credit prediction. A randomized subset of transitions
receives two-sided continuation estimates, and a model-assisted correction
removes the predictor's conditional bias in expectation.
The rubric is thus a prediction representation; it does not replace the
environment's definition of success.

The central design criterion follows from the variance of this corrected
estimator. Prediction error matters most where a decision has a large
policy-score norm and little replay coverage. This yields a held-out,
variance-weighted objective for adding, removing, or merging criteria.
The same decomposition yields a square-root allocation of replay probabilities
that balances residual uncertainty, score sensitivity, and continuation cost.
These are coupled design choices within one conditional variance objective,
rather than two independently chosen heuristics.

Our contributions are a concrete framework for this coupling, a precise account
of its estimator guarantees and implementation boundaries, and a falsifiable
evaluation protocol. The inverse-probability correction and allocation principle
build on established model-assisted estimation and variance-reduction ideas
\citep{horvitz1952generalization,sarndal1992model,gu2017qprop,grathwohl2018backpropagation};
we do not claim those statistical identities as new.
This version includes an exact finite-model check, but no completed comparison
showing that evolving rubrics outperform ordinary predictors or additional
original rollouts at equal cost.

\section{Related work}
\paragraph{Outcome-based and grouped optimization.}
REINFORCE estimates policy gradients from returns \citep{williams1992simple}.
RLOO uses other independently sampled responses as a baseline
\citep{ahmadian2024back}; GRPO additionally uses within-group normalization
\citep{shao2024deepseekmath}. GiGPO forms groups at repeated intermediate
states to obtain local comparisons without additional branching
\citep{feng2025gigpo}. Evidence-calibrated optimization studies the reliability
of sparse local comparisons \citep{li2026evidence}.
\credo{} actively acquires new continuation evidence when the required simulator
operations are available. Its terminal component uses unstandardized RLOO;
it should not be labeled numerically identical to normalized GRPO.

\paragraph{Rubrics and counterfactual continuations.}
DRACO uses dynamic rubrics and step citations for credit allocation in an
outcome-blind training regime \citep{gandhi2026draco}. Our setting requires a
terminal verifier and resettable state, so the methods have different information
and environment assumptions. A comparison must disclose whether the original
DRACO regime is retained or adapted.
BPO exploits sandbox branching \citep{he2026branching}, while CVT-RL considers
policy-conditioned counterfactual credit under explicitly specified interventions
\citep{meng2026policy}.
Our intervention keeps the realized transition on one side and resamples from
its preceding state on the other. It does not delete an action or establish
policy-independent causal necessity.

\paragraph{Model-assisted estimation and control variates.}
Known-probability sampling can correct a prediction by an inverse-probability
weighted residual \citep{horvitz1952generalization,sarndal1992model}.
Learned control variates can reduce policy-gradient variance
\citep{gu2017qprop,grathwohl2018backpropagation}.
Here the inclusion probabilities are designed and recorded, not estimated
propensities. The proposed research claim is that semantic feature evolution
and execution allocation can be usefully coordinated by the resulting
variance objective. Whether that representation is preferable to a conventional
text-conditioned predictor remains an empirical question.

\section{Setting and method}
\subsection{Terminal objectives and realized-transition credit}
Let $z$ denote a task, $h_t$ the history visible to the policy, and $\xi_t$ the
complete restorable state immediately before action $a_t$.
The augmented state includes the environment, history, relevant random state,
clock, and remaining interaction budget. The policy is
$\pi_\theta(a_t\mid h_t)$, and a trajectory receives a bounded terminal reward
$R(\tau)\in[0,1]$. We use binary verified success in the proposed agent experiments.
The undiscounted objective and action score are
\begin{equation}
 J(\theta)=\E_{z,\tau\sim\pi_\theta}[R(\tau)],
 \qquad g_t=\nabla_\theta\log\pi_\theta(a_t\mid h_t).
 \label{eq:objective}
\end{equation}
An action can be a variable-length code block. Its score is the sum of the
scores of its sampled action tokens, not the score of a re-tokenized
normalization of the code.

For a realized transition $x_t=(\xi_t,a_t,\xi_{t+1})$, define
\begin{equation}
 U_\theta(x_t)=\E_{\pi_\theta}[R\mid \xi_{t+1}],
 \quad V_\theta(\xi_t)=\E_{\pi_\theta}[R\mid \xi_t],
 \quad C_\theta(x_t)=U_\theta(x_t)-V_\theta(\xi_t).
 \label{eq:credit}
\end{equation}
For a terminal post-state, $U_\theta$ is its verified reward. In stochastic
environments, $C_\theta$ conditions on the transition that actually occurred;
averaging over that transition yields the usual action advantage
$Q_\theta(\xi_t,a_t)-V_\theta(\xi_t)$. This distinction is needed for the
policy-gradient argument and for interpreting individual replay labels.
With terminal-only rewards, this target is an exact-value temporal-difference
increment on the augmented state, rather than a new definition of advantage.

\subsection{Two-sided replay}
Freeze the sampling policy $\pi_{\theta_k}$ for a complete batch.
At a selected transition, restore its post-state and draw $M$ fresh continuations
with rewards $R^+_1,\ldots,R^+_M$. Independently restore the pre-state for each of
$M$ baseline continuations, resample the first action from $\pi_{\theta_k}$,
and obtain $R^-_1,\ldots,R^-_M$. The baseline is allowed to sample the original
action again. Define
\begin{equation}
 \Delta_t=\frac{1}{M}\sum_{m=1}^M R^+_m-
                 \frac{1}{M}\sum_{m=1}^M R^-_m.
 \label{eq:replay}
\end{equation}
Both sides use the original remaining horizon: the post-state has one fewer
available action than the pre-state. Neither side reuses the original
trajectory's terminal reward as a fresh continuation sample.
For independent binary returns,
\begin{equation}
 \E[\Delta_t\mid x_t]=C_t,\qquad
 \sigma_t^2:=\Var(\Delta_t\mid x_t)
   =\frac{U_t(1-U_t)+V_t(1-V_t)}{M}.
 \label{eq:replaynoise}
\end{equation}
The terminal actual side has zero variance and need not execute additional
actions. State reconstruction and any retries remain part of replay cost.

\subsection{Rubric prediction and randomized correction}
A rubric $\mathcal R_k$ contains $K$ criteria. A frozen judge evaluates all
criteria using only the task, visible prefix, action, and observed feedback.
It does not receive hidden simulator state, hidden tests, or replay outcomes.
Let $s^-_t,s^+_t\in\mathbb R^K$ be pre- and post-action scores. We construct
\begin{equation}
 \phi_{\mathcal R_k}(x_t)=[s^-_t;\ s^+_t-s^-_t],
 \qquad c_t=b_k+w_k^\top\phi_{\mathcal R_k}(x_t).
 \label{eq:features}
\end{equation}
The linear head is unconstrained: neither the predictor nor its learned feature
weights is required to be a potential difference. More expressive heads are
compatible with the estimator, but need their own cost and capacity comparisons.

Before observing a fresh replay label, assign a known probability
$p_t\in[p_{\min},1]$ and independently draw $S_t\sim\operatorname{Bernoulli}(p_t)$.
Only selected positions are replayed. The corrected credit is
\begin{equation}
 \widehat C_t=c_t+\frac{S_t}{p_t}(\Delta_t-c_t).
 \label{eq:correction}
\end{equation}
The second term is evaluated only when $S_t=1$.
Rubric, head, allocation rule, and policy version are frozen for the batch.
They may depend on earlier training data, but not on the fresh replay randomness
they are about to correct. The predictor can be inaccurate; accuracy changes
variance rather than the conditional expectation.

\begin{figure}[t]
 \centering
 \includegraphics[width=\linewidth]{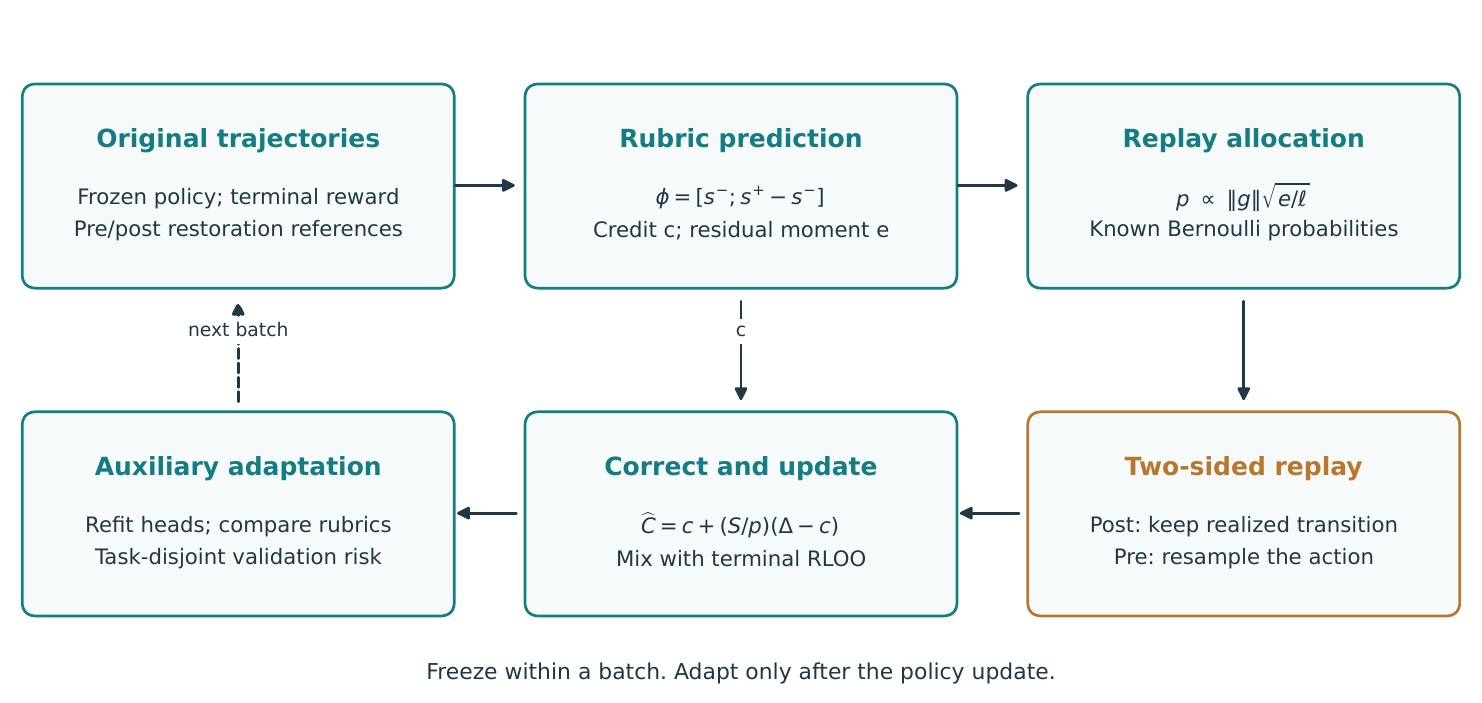}
 \caption{One frozen batch. Rubric features predict every original transition.
 Randomly selected transitions receive independent two-sided replay, correcting
 only their prediction residuals. After the policy update, replay data inform
 the next credit/error heads and held-out rubric selection. Replay continuations
 supply labels; they are not silently added to the original policy-update batch.}
 \label{fig:method}
\end{figure}

\subsection{A practical terminal-credit mixture}
For each task, collect $G>1$ independent original trajectories. The terminal
leave-one-out advantage is
\begin{equation}
 A_i^{\mathrm{out}}
 =R_i-\frac{1}{G-1}\sum_{j\ne i}R_j,
 \qquad
 A_{i,t}^{\mathrm{mix}}
 =(1-\alpha)A_i^{\mathrm{out}}+\alpha\widehat C_{i,t}.
 \label{eq:mixture}
\end{equation}
Only other original trajectories enter this baseline. Replay branches are kept
separate. We choose a batch-independent constant $\alpha\in[0,1]$; a starting
configuration is $\alpha=0.5$. The endpoints isolate the terminal and corrected
credit signals. Both components retain the terminal reward's original scale,
without separate group standardization. At $\alpha=0$, rubric scoring and
replay can be disabled entirely.

For the practical PPO variant \citep{schulman2017proximal}, broadcast the
stopped-gradient action advantage to its sampled tokens and use
\begin{equation}
 \rho_{i,t,u}(\theta)
 =\frac{\pi_\theta(y_{i,t,u}\mid h_{i,t},y_{i,t,<u})}
        {\pi_{\theta_k}(y_{i,t,u}\mid h_{i,t},y_{i,t,<u})}.
\end{equation}
The clipped loss is the masked aggregation of
$-\min\{\rho A^{\mathrm{mix}},
\clip(\rho,1-\epsilon,1+\epsilon)A^{\mathrm{mix}}\}$, with the usual optional
KL regularizer to a fixed reference. Tool feedback, formatting delimiters
inserted by the runtime, and padding receive zero loss mask.
All labels, probabilities, and auxiliary predictions are detached.
The exact gradient result below concerns the unnormalized action-score
estimator at $\theta_k$; it does not assert that repeated clipped PPO updates
or per-trajectory token averaging remain unbiased for Eq.~\eqref{eq:objective}.

\subsection{Replay allocation}
Let $e_t$ estimate the second moment
\begin{equation}
 e_t\approx \E[(\Delta_t-c_t)^2\mid x_t]
       =\sigma_t^2+(C_t-c_t)^2,
 \label{eq:error}
\end{equation}
and let $\ell_t>0$ estimate the cost of \emph{both} replay sides, including
state reconstruction. The squared policy-score norm $\|g_t\|^2$ measures
local gradient sensitivity. For fixed predictor and calibrated inputs,
the variance objective gives the allocation
\begin{equation}
 p_t^\star=
 \clip_{[p_{\min},1]}
 \left(\sqrt{\frac{\|g_t\|^2e_t}{\lambda\ell_t}}\right),
 \qquad \sum_t p_t^\star\ell_t\le B.
 \label{eq:allocation}
\end{equation}
A scalar search selects $\lambda$ for the desired expected cost.
$B$ is the budget available for replay after other costs are accounted for;
this allocation does not jointly optimize judge or rubric-search overhead.
A target average inclusion rate is a different constraint and must not be
reported as equal compute when continuation costs vary.
Independent Bernoulli sampling does not enforce an exact realized budget:
discarding selected jobs after sampling invalidates the recorded probabilities.

Exact full-parameter score norms are expensive for large models. Approximations
may be substituted as \emph{allocation heuristics}, with their definition and
cost reported. Known positive inclusion probabilities still permit the
correction in Eq.~\eqref{eq:correction}; oracle allocation optimality does not
transfer automatically to approximate norms or learned error/cost models.
Approximating the allocation norm does not replace the policy score used
in the gradient estimator.

\subsection{Learning the heads and evolving the rubric}
Archive each acquired label with its original policy version, old prediction,
inclusion probability $q_t$, and feature context.
When the only acquisition stream is the correction sampler, $q_t=p_t$.
Let $p_t^{\mathrm{ref}}$ be a fixed reference allocation shared by all candidates
in a rubric comparison. Define
\begin{equation}
 w_t^{\mathrm{var}}=\|g_t\|^2\left(\frac{1}{p_t^{\mathrm{ref}}}-1\right),
 \qquad
 \widehat{\mathcal L}_{\mathcal R}
 =\sum_{t\in\mathcal A}\frac{w_t^{\mathrm{var}}}{q_t}
       \bigl(\Delta_t-c_{\mathcal R}(x_t)\bigr)^2 .
 \label{eq:rubricrisk}
\end{equation}
Fit each candidate head on fitting tasks, with a ridge penalty on its weights.
Then compare candidates on the \emph{same held-out tasks and replay labels}.
Candidate proposals may add, remove, or merge criteria using large residuals
from the fitting partition only. Neither the held-out labels nor official
benchmark test tasks may guide candidate generation.
Both auxiliary partitions are drawn from training tasks; official benchmark
Dev tasks are reserved for experiment-level tuning and checkpoint selection.
Accept a candidate only if its held-out risk improves by a prespecified margin;
then update the allocation for the next batch.
A common weight normalization preserves candidate rankings, but a
self-normalized risk should not be described as a finite-sample unbiased
Horvitz--Thompson estimate.

The nonnegative error head is trained by inverse-inclusion weighted regression
on $(\Delta_t-c_t^{\mathrm{old}})^2$, using the prediction stored \emph{before}
fitting to that label. This target includes replay noise and does not multiply
in $\|g_t\|^2$ a second time. It avoids artificially small in-sample residuals,
but does not guarantee calibration after the credit head or policy changes.
Versioning, held-out error calibration, and optional cross-fitting are therefore
part of the evaluation, not assumptions that a learned head satisfies by design.

\begin{figure}[t]
\small
\fbox{\begin{minipage}{0.94\linewidth}
\textbf{Algorithm 1: One CREDO batch}
\begin{enumerate}[leftmargin=1.5em]
\item Freeze the policy, rubric, judge, credit/error heads, and cost model.
\item Collect $G$ original trajectories per task; record rewards, old token
      probabilities, masks, and pre/post restoration references.
\item Score visible transitions, predict $c_t,e_t$, and assign $p_t$ before
      drawing independent replay indicators.
\item Execute both replay sides for each selected position with fresh random
      streams and the original remaining horizon.
\item Form $\widehat C_t$ and $A^{\mathrm{mix}}_{i,t}$; finish the policy update
      on original trajectories only.
\item Update the auxiliary archive and heads. Propose rubric candidates from
      fitting tasks and compare them on a fixed, task-disjoint validation set.
\item Publish the accepted auxiliary versions for the next batch and checkpoint
      policy, optimizer, rubric, heads, and acquisition metadata together.
\end{enumerate}
\end{minipage}}
\caption{Policy optimization precedes auxiliary adaptation on the same labels.
Any replay retry retains its original sampling decision and records actual cost.}
\label{alg:credo}
\end{figure}

\section{Estimator properties and the shared design objective}
\label{sec:theory}
\begin{assumption}[Frozen, valid acquisition]
\label{ass:valid}
Condition on the original batch $\mathcal B$ and all previously fitted
auxiliary components. Scores, predictions, and probabilities are fixed.
For each position, $p_t>0$ is its true inclusion probability;
$S_t$ is independent of its fresh replay label; and restoration with the
remaining horizon produces $\E[\Delta_t\mid\mathcal B]=C_t$.
For the additive covariance formula, acquisition indicators and replay
randomness are also independent across positions.
\end{assumption}
Future random streams must follow the same conditional continuation law as
the target policy and environment. Resampling a persistent hidden task
variable, changing a tool response, or resetting the time budget would change
the target and violate the assumption.

\begin{proposition}[Conditional correction]
\label{prop:correction}
Under Assumption~\ref{ass:valid}, the estimator in
Eq.~\eqref{eq:correction} satisfies
$\E[\widehat C_t\mid\mathcal B]=C_t$ for every fixed predictor $c_t$.
\end{proposition}
\begin{proof}
Independence and $\E[S_t\mid\mathcal B]=p_t$ give
$\E[\widehat C_t\mid\mathcal B]
=c_t+p_t(C_t-c_t)/p_t=C_t$.
\end{proof}
This is a design-based statement. It does not require a correct rubric or error
model, but it does require the correct sampling probability and valid replay.

\begin{proposition}[Terminal-objective gradient]
\label{prop:gradient}
Suppose trajectories have finite horizon, differentiation can pass through
expectation, the environment dynamics are independent of $\theta$, and
$g_t$ is the score in Eq.~\eqref{eq:objective}.
Under Assumption~\ref{ass:valid},
\begin{equation}
 \E\!\left[\sum_t g_t\widehat C_t\right]=\nabla_\theta J(\theta).
 \label{eq:unbiasedgradient}
\end{equation}
For $G>1$ independent original trajectories per task and constant $\alpha$,
$G^{-1}\sum_{i,t}g_{i,t}A^{\mathrm{mix}}_{i,t}$ has the same expectation.
\end{proposition}
The proof in Appendix~\ref{app:proofs} uses both conditioning on the realized
transition and cancellation of a pre-action baseline. It does not require
deterministic transitions. A fixed linear mixture preserves the ideal
gradient expectation; this alone does not show that it reduces variance.

\begin{proposition}[Conditional replay variance]
\label{prop:variance}
Write $r_t=C_t-c_t$ and $\sigma_t^2=\Var(\Delta_t\mid\mathcal B)$. Then
\begin{equation}
 \Var(\widehat C_t\mid\mathcal B)
 =\frac{\sigma_t^2}{p_t}
  +\left(\frac{1}{p_t}-1\right)r_t^2.
 \label{eq:variance}
\end{equation}
For the averaged, unnormalized mixed gradient $\widehat G$, independent
acquisition across positions implies
\begin{equation}
 \tr\Cov(\widehat G\mid\mathcal B)
 =\frac{\alpha^2}{G^2}\sum_{i,t}\|g_{i,t}\|^2
 \left[\frac{\sigma_{i,t}^2}{p_{i,t}}+
       \left(\frac{1}{p_{i,t}}-1\right)r_{i,t}^2\right].
 \label{eq:gradientvariance}
\end{equation}
\end{proposition}
Original trajectory positions can be correlated: they are fixed in this
conditional statement. Their contribution reappears in
$\Cov(\E[\widehat G\mid\mathcal B])$ under the law of total covariance.
Shared baseline rollouts or coupled replay random numbers introduce additional
cross-covariance terms and are outside Eq.~\eqref{eq:gradientvariance}.

At $p_t=1$, the predictor has no effect on conditional replay variance.
At smaller $p_t$, its squared error is weighted by missing coverage
$(1/p_t-1)$. Thus, for fixed reference probabilities, minimizing
Eq.~\eqref{eq:rubricrisk} targets the predictor-dependent part of this variance:
\begin{equation}
 \E[(\Delta_t-c_{\mathcal R}(x_t))^2\mid x_t]
 =\sigma_t^2+(C_t-c_{\mathcal R}(x_t))^2.
\end{equation}
The noise term is common across candidates when policy, label construction,
and evaluation points are fixed. This identity justifies the comparison
criterion, not a claim that finite held-out selection always improves the
population risk.

\begin{proposition}[Oracle allocation at fixed expected cost]
\label{prop:allocation}
For fixed $c_t$, exact $e_t=\sigma_t^2+r_t^2$, and positive fixed costs $\ell_t$,
minimizing Eq.~\eqref{eq:gradientvariance} subject to
$p_{\min}\le p_t\le1$ and $\sum_t p_t\ell_t\le B$ gives
Eq.~\eqref{eq:allocation} on nonzero-benefit positions, with the usual
boundary cases for a slack budget. Feasibility requires
$B\ge p_{\min}\sum_t\ell_t$.
\end{proposition}
The objective reduces to $\sum_t\|g_t\|^2 e_t/p_t$ plus a constant.
Its interior derivative is
$-\|g_t\|^2 e_t/p_t^2+\lambda\ell_t$, yielding the square-root rule.
This is conditional allocation optimality, not optimal policy learning or a
guarantee about estimated costs.

\subsection{Exact finite-model check}
We enumerate a two-decision stochastic model with 16 original trajectories,
all Bernoulli inclusion decisions, and all two-sided replay outcomes for $M=2$.
The scalar objective gradient is $0.1127514540$. Table~\ref{tab:audit} is
generated by the accompanying Python script, not fitted to observed LLM results.
The no-correction controls retain the same expected replay budget but ignore
the labels in their update; oracle variants are diagnostic constructions.

\begin{table}[t]
\centering\small
\begin{tabular}{lrrrr}
\toprule
Estimator & $\E[\widehat G]$ & $|\mathrm{bias}|$ & $\Var(\widehat G)$ & Cost\\
\midrule
Zero + correction & 0.112751 & 0.000000 & 1.668019 & 0.80 \\
Reversed oracle + correction & 0.112751 & 0.000000 & 3.881736 & 0.80 \\
Oracle + correction & 0.112751 & 0.000000 & 0.930114 & 0.80 \\
Zero + oracle allocation & 0.112751 & 0.000000 & 1.319479 & 0.80 \\
Zero, no correction & 0.000000 & 0.112751 & 0.000000 & 0.80 \\
Reversed oracle, no correction & -0.112751 & 0.225503 & 0.073210 & 0.80 \\
\bottomrule
\end{tabular}

\caption{Exact finite-model audit. Cost is the expected number of additional
environment actions per original trajectory, excluding the common original
rollout. Corrected estimators match the analytic gradient to numerical precision;
bad predictions still increase variance. Oracle allocation reduces variance
relative to uniform allocation for the same zero predictor and expected cost.
These calculations do not establish the quality of a learned rubric or allocator.}
\label{tab:audit}
\end{table}

\section{Evaluation protocol and present evidence}
\label{sec:evaluation}
Full language-agent comparisons remain to be completed. This version reports no
\credo{} benchmark improvement and does not attribute results from a separate
outcome-only GRPO training run to \credo{}.
The finite-model calculation verifies algebra and implementation of the
estimator; it cannot establish practical sample efficiency.

\paragraph{Planned settings.}
The current experimental plan uses AppWorld as a primary resettable coding
environment \citep{trivedi2024appworld}, with a second conversational
tool-use setting based on $\tau^2$-Bench \citep{barres2025tau2}.
ALFWorld \citep{shridhar2020alfworld} is intended for lower-cost mechanism
experiments. For held-out generalization, we plan to evaluate the multi-turn
tasks of BFCL \citep{patil2025bfcl} using a pinned v4 release
\citep{mao2025bfclv4}. Multi-turn evaluation was introduced in BFCL v3;
the release and exact task subset will be reported separately from v4's
additional agentic categories. Only the AppWorld replay adapter is presently available;
the other integrations remain planned.
The principal model scale is Qwen3-8B, with Qwen3-4B as a secondary scale
\citep{yang2025qwen3}.
All compared methods must share initialization, prompt protocol, action unit,
interaction limits, terminal evaluator, and evaluation decoding.
Benchmark releases and task manifests will be fixed before comparison;
later versions of the $\tau$-Bench suite will not be silently substituted
for the cited $\tau^2$-Bench setting.

For AppWorld, train on the official Train split, tune on Dev, and reserve
Test-Normal and Test-Challenge for fixed final evaluation.
The initial shared AppWorld budget is 50 actions, a 32,768-token context,
and at most 2,048 generated tokens per action; any revised setting must be
fixed across methods before final evaluation.
Task Goal Completion (TGC) is the fraction of tasks passing the official
task evaluator. Scenario Goal Completion (SGC) requires all associated
tasks in a scenario to succeed. Action execution success and the fraction
of hidden checks passed are separate diagnostics, not SGC.

\paragraph{Comparisons needed to support the method.}
The principal controls are outcome-only RLOO and normalized GRPO; additional
original trajectories at matched total cost; a conventional text-conditioned
predictor with identical replay labels, correction, and allocation; and
compatible branching and rubric baselines.
The representation claim requires beating the matched conventional predictor,
not merely an outcome-only method.
Compare fixed versus evolving rubrics and uniform versus adaptive replay in
a $2\times2$ design. Also compare variance-weighted rubric selection with
ordinary MSE and terminal-outcome discrimination.
The no-correction ablation still acquires the same labels to avoid confounding
the update rule with the supervision budget.

\paragraph{Mechanism and cost measurements.}
At fixed policy checkpoints, measure gradient bias against exact truth in small
models, and repeat acquisition to estimate conditional gradient covariance.
In real environments, high-budget replay is a noisy reference, not ground truth.
Measure the relative magnitudes of $\sigma_t^2$ and $(C_t-c_t)^2$, allocator
calibration, importance-weight tails, and predictor drift as policies change.
Sweep $M$ under a fixed \emph{total} budget rather than treating replay count as
a free accuracy knob.

Report original, replay, judge, candidate-scoring, and reconstruction tokens;
environment calls; and reserved GPU-hours separately. Equal selection rates
do not imply equal cost. Budget-matched success curves and the cost to reach
a fixed success level are the primary efficiency evidence.
Use multiple independent training seeds and task/scenario-aware uncertainty
estimates. Repeated continuations from one state are repeated measurements,
not additional independent tasks. With few seeds, interaction effects in the
$2\times2$ design should be treated as exploratory.

\section{Limitations and failure modes}
\paragraph{Restoration and evaluator dependence.}
The framework applies to controlled training sandboxes with a trustworthy
terminal reward. Database restoration alone is insufficient when execution
also depends on a Python namespace, user simulator, clock, or remaining budget.
It is not intended to replay irreversible operations against production users.
Invalid reconstruction biases the replay label; inverse-probability weighting
cannot repair that bias.

\paragraph{Variance and computational overhead.}
Unbiasedness does not imply low variance or low mean-squared error.
Small $p_t$ amplifies continuation noise as well as prediction error.
If $\sigma_t^2$ dominates, improving rubrics has little effect on total variance.
Under a fixed budget, increasing $M$ can reduce coverage, so the useful operating
point must be measured. Scoring every transition and computing score norms
may also outweigh any saved replay; all such costs belong in the comparison.
Clipping corrected credit values changes the expectation guarantee and must
be identified as a separate practical modification.

\paragraph{Adaptive selection and predictor drift.}
Sparse acquisition may miss a systematically wrong criterion. Positive floors
provide support but not rapid discovery. A separate uniform diagnostic stream
is a possible extension, with its true union probability recorded if labels
also contribute to correction. Repeated rubric selection on the same held-out
tasks can overfit them. Updated policies and heads can make error predictions
stale. None of these problems is removed by the conditional expectation identity.

\paragraph{Optimization and empirical scope.}
Token-length normalization, PPO and gradient clipping, multiple optimizer epochs, regularization,
and approximate score norms create a gap between the ideal estimator and the
practical update. We do not claim an unbiased full training procedure,
monotone learning, or superiority over a learned critic. Those claims would
require additional analysis or controlled empirical evidence.

\section{Conclusion}
\credo{} treats semantic rubrics as revisable credit predictors and execution
replay as a selective correction mechanism. A conditional variance criterion
links how rubric features are evaluated to how replay is allocated.
The estimator has a clear expectation guarantee under explicit restoration
and sampling assumptions, while its practical value depends on predictor
quality, continuation noise, and total cost. The next empirical test is whether
this coupling improves verified task performance over matched conventional
predictors and additional original sampling.

\section*{Use of generative AI tools}
OpenAI Codex assisted with drafting and revising the manuscript, preparing
the LaTeX source, and developing the finite-model audit code.
The numerical audit values were computed by the accompanying executable
Python program. The author retains responsibility for the scientific content,
references, and accompanying code.

{\small
\setlength{\bibsep}{3pt}
\bibliographystyle{plainnat}
\bibliography{references}
}
\clearpage
\appendix
\section{Proof details}
\label{app:proofs}
\subsection{Policy-gradient identity}
For a fixed task and frozen auxiliary history, the likelihood-ratio identity
gives $\nabla_\theta J=\E[\sum_t g_tR]$.
Because the realized augmented transition determines the conditional continuation
law and $g_t$ is measurable with respect to it,
\begin{equation}
 \E[g_tU_\theta(x_t)]=\E[g_tR].
\end{equation}
Since $\xi_t$ contains $h_t$ and precedes the action,
\begin{equation}
 \E[g_tV_\theta(\xi_t)]
 =\E\!\left[V_\theta(\xi_t)
       \sum_a\pi_\theta(a\mid h_t)\nabla_\theta\log\pi_\theta(a\mid h_t)\right]=0.
\end{equation}
Thus $\E[\sum_tg_tC_t]=\nabla_\theta J$. Proposition~\ref{prop:correction}
then gives Proposition~\ref{prop:gradient}.
For original trajectory $i$, its leave-one-out baseline depends on other,
conditionally independent trajectories. Conditioning on those trajectories
leaves the zero-score identity unchanged. Consequently
$\E[\sum_tg_{i,t}A_i^{\mathrm{out}}]=\nabla_\theta J$ as well.
Linearity proves the statement for a constant mixture.
Choosing a data-dependent per-transition mixture, dividing by random action
lengths, or evaluating the score under a different policy is not justified by
this argument.

\subsection{Variance and candidate risk}
Suppress the position index and write $D=\Delta-c$, so that
$\E[D]=r$ and $\E[D^2]=\sigma^2+r^2$. Independence and $S^2=S$ imply
\begin{equation}
 \E[\widehat C^2\mid\mathcal B]
 =c^2+2cr+\frac{\sigma^2+r^2}{p}.
\end{equation}
Subtracting $C^2=(c+r)^2$ proves Eq.~\eqref{eq:variance}.
For a fixed vector $g$,
$\Cov(g\widehat C\mid\mathcal B)=gg^\top\Var(\widehat C\mid\mathcal B)$.
Its trace is $\|g\|^2\Var(\widehat C\mid\mathcal B)$.
Conditionally independent replay errors yield the sum in
Eq.~\eqref{eq:gradientvariance}; the terminal part is constant given
$\mathcal B$.

For a finite held-out frame of positions, let $I_t$ indicate archive inclusion
with probability $q_t$. Conditional on fixed candidates fitted elsewhere,
\begin{equation}
 \E\!\left[\sum_t\frac{I_t}{q_t}w_t^{\mathrm{var}}
     (\Delta_t-c_{\mathcal R}(x_t))^2\right]
 =\sum_t w_t^{\mathrm{var}}\!
     \left[\sigma_t^2+(C_t-c_{\mathcal R}(x_t))^2\right].
\end{equation}
All candidates must share labels, $q_t$, score norms, and reference probabilities.
Dividing each empirical candidate score by the same observed total weight leaves
its ordering unchanged. It does not make that ratio an unbiased risk estimate.
Likewise, picking the smallest observed validation loss does not imply a
population improvement when the candidate set or validation reuse is large.

\subsection{Budget allocation}
Set $a_t=\|g_t\|^2e_t$. Ignoring constants independent of $p$, the objective is
$\sum_t a_t/p_t$. Its second derivative is $2a_t/p_t^3\ge0$.
The KKT conditions for positive $a_t$ and an active budget give
$a_t/p_t^2=\lambda\ell_t$ at interior positions, with clipping at the bounds.
The expected-cost sum decreases monotonically with $\lambda$, enabling bisection.
Zero-benefit positions may remain at the probability floor, and excess budget
need not be spent once all beneficial positions saturate.
Optimizing estimated quantities only approximately follows this oracle rule.

\section{Exact audit specification}
\label{app:audit}
The audit has two decisions, binary actions, one intermediate binary state
$s\in\{0,1\}$, and a Bernoulli terminal reward. The probability of action one is
$\operatorname{sigmoid}(\theta f)$ with $\theta=0.3$ and
$f_{\mathrm{start}}=1$, $f_0=-0.6$, $f_1=1.4$.
The intermediate state is one with probability $0.8$ after action one and
$0.2$ after action zero. Terminal success probabilities are
\begin{equation}
\begin{array}{c|cc}
 & a=0 & a=1\\ \hline
 s=0 &0.10&0.55\\
 s=1 &0.45&0.90
\end{array}.
\end{equation}
Enumerating $(a_0,s,a_1,R)$ produces 16 original paths and
$J(0.3)=0.5318009108$.
The action score is $f(a-\operatorname{sigmoid}(\theta f))$.
At each position, the two replay means follow scaled binomial distributions
with $M=2$. Uniform acquisition uses $p=0.1$ and the adaptive audit uses the
exact residual second moments, floor $0.02$, and the same expected action
budget $0.8$. Replay costs are six actions at the first position and two at
the terminal transition. State-copy and judge costs are absent in this
mathematical model and are not being estimated for a real system.

The script enumerates replay outcomes and selection indicators to compute
gradient moments, then independently checks the variance decomposition.
Its oracle and reversed-oracle predictors are deliberate diagnostics, not
learned models. For no-correction controls, acquired labels are ignored while
the same expected replay work is charged. The source is
\texttt{anc/verify\_estimator.py} in the source archive; numeric output is
\texttt{anc/exact\_audit.json}.

\section{Implementation contract and starting configuration}
An AppWorld action is one generated and executed Python code block, regardless
of how many API calls appear inside it. Restoration references must account
for databases, Python execution context, history, clocks, and remaining horizon.
Operational RNG metadata should support deterministic reconstruction checks;
fresh continuation streams must sample the target conditional dynamics without
resampling persistent task facts. Non-serializable execution objects can require
reconstructing the recorded prefix in a fresh process and checking state
fingerprints. Every reconstruction call belongs in the cost ledger.

A starting configuration uses $G=8$, $\alpha=0.5$, $M=2$ per replay side,
mean inclusion rate $10\%$, probability floor $2\%$, eight initial rubric
criteria, at most sixteen criteria, and up to four candidates per revision.
These are development starting points, not empirically established optima.
Budget-matched experiments must replace the inclusion-rate target by an
explicit expected-cost constraint.

For a frozen, state-scoring judge, cached post-state scores can be reused as
the next pre-state scores only when the rubric, judge version, and scoring
context match exactly. Such caching is an implementation option, not an
assumption in the cost claims.
An optional independent diagnostic acquisition stream with probability $d_t$
has union probability $q_t=1-(1-p_t)(1-d_t)$.
If union-acquired labels also enter the correction, the denominator must be
$q_t$, not the adaptive stream's $p_t$. If they are used only for auxiliary
learning, gradient correction retains its own acquisition law.

Finally, transient replay infrastructure failures must not be converted to
terminal reward zero or silently removed from the selected set.
Retain acquisition identities, retry under a documented protocol, and stop
the affected training update if a valid selected label cannot be obtained.
This requirement is distinct from a successfully executed continuation whose
official task reward is zero.

\end{document}